\documentclass{article}

\usepackage[preprint]{neurips_2025}

\usepackage{comment}
\usepackage[utf8]{inputenc} 
\usepackage[T1]{fontenc}    
\usepackage{hyperref}       
\usepackage{url}            
\usepackage{booktabs}       
\usepackage{amsfonts}       
\usepackage{nicefrac}       
\usepackage{microtype}      
\usepackage{xcolor}         
\usepackage{amssymb}            
\usepackage{mathtools}          
\usepackage{mathrsfs}           
\usepackage{graphicx}           
\usepackage{subcaption}         
\usepackage[space]{grffile}     
\usepackage{url}                
\usepackage{lipsum}             
\usepackage{bbm}
\usepackage{amsmath}
\usepackage{amsthm}
\usepackage{algorithm}
\usepackage{algorithmic}
\usepackage{enumitem}
\newtheorem{theorem}{Theorem}[section]

\newtheorem{lemma}[theorem]{Lemma}

\newtheorem{assumption}[theorem]{Assumption}
\newtheorem{remark}[theorem]{Remark}
\DeclareMathOperator*{\argmin}{argmin} 
\DeclareMathOperator*{\argmax}{argmax}
\DeclareMathOperator*{\E}{\mathbb{E}}
\DeclareMathOperator{\ba}{\boldsymbol{a}}
\DeclareMathOperator{\bs}{\boldsymbol{s}}
\DeclareMathOperator{\bP}{\boldsymbol{P}}
\DeclareMathOperator{\bR}{\boldsymbol{R}}

\DeclareMathOperator{\blambda}{\mathbf{\lambda}}

\newcommand{\norm}[1]{\left\lVert#1\right\rVert_1}
\newcommand{\abs}[1]{\left\lvert#1\right\rvert}
\newcommand{\rad}[1]{2|S|\log(|S||A|#1/\delta)}
\newcommand{\upV}{\frac{N(1 + U_{\lambda})}{1-\gamma}}

\newcommand{\PR}[1]{\text{Pr}\left(#1\right)}

\newcommand{\rmab}{NS-Whittle}
\newcommand{\ie}{\textit{i.e., }}

\title{Non-Stationary Restless Multi-Armed Bandits with Provable Guarantee}

\author{%
  Yu-Heng Hung\thanks{\href{}{https://hungyuheng.github.io/yuheng/}} \\
  Department of Computer Science\\
  National Yang Ming Chiao Tung University, Hsinchu, Taiwan \\
  \texttt{hungyh.cs08@nycu.edu.tw} \\
  \And
  Ping-Chun Hsieh \\
  Department of Computer Science \\
  National Yang Ming Chiao Tung University, Hsinchu, Taiwan \\
  \texttt{pinghsieh@nycu.edu.tw} \\
  \AND
  Kai Wang \\
  Georgia Institute of Technology \\
  \texttt{kwang692@gatech.edu} \\
}

\begin{document}

\maketitle

\begin{abstract}
Online restless multi-armed bandits (RMABs) typically assume that each arm follows a stationary Markov Decision Process (MDP) with fixed state transitions and rewards. However, in real-world applications like healthcare and recommendation systems, these assumptions often break due to non-stationary dynamics, posing significant challenges for traditional RMAB algorithms. In this work, we specifically consider $N$-armd RMAB with non-stationary transition constrained by bounded variation budgets $B$. Our proposed \rmab\; algorithm integrates sliding window reinforcement learning (RL) with an upper confidence bound (UCB) mechanism to simultaneously learn transition dynamics and their variations. We further establish that \rmab\; achieves $\widetilde{\mathcal{O}}(N^2 B^{\frac{1}{4}} T^{\frac{3}{4}})$ regret bound by leveraging a relaxed definition of regret, providing a foundational theoretical framework for non-stationary RMAB problems for the first time.
\end{abstract}

\section{Introduction}
Restless multi-armed bandits (RMABs) provide a powerful framework for sequential decision-making in dynamic environments. An RMAB instance consists of $N$ arms, each evolving independently as a Markov decision process (MDP), whose state continues to evolve according to its underlying dynamics, irrespective of whether it is chosen. Unlike traditional multi-armed bandits (MABs), RMABs allow for non-i.i.d. reward distributions that depend on time-varying states, making them well-suited for real-world applications such as queuing systems~\cite{ansell2003whittle,nino2002dynamic}, healthcare interventions~\cite{mate2022field,mate2020collapsing}, and wireless networks~\cite{wei2010distributed,sun2018cell}. However, in practical scenarios, these dynamics are often unknown a priori, transforming RMABs into an online learning problem requiring a balance between exploration and exploitation with regret guarantees. 

Recent work has made strides in developing algorithms for online RMABs under the assumption of stationary dynamics~\cite{wang2023optimistic}. However, this assumption often fails to hold in real-world settings. For example, in many real-world applications, such as healthcare, patient responses to treatments may evolve over time, while external factors like changing user preferences or environmental influences introduce non-stationarity. This necessitates the development of algorithms that can handle non-stationary dynamics. 

To address the non-stationary dynamics, \cite{cheung2020reinforcement} proposed a sliding window approach to handle non-stationary MDPs. A straightforward idea is to treat the entire RMAB problem as a single large MDP and apply the same method. While this approach simplifies the learning framework, it introduces significant computational challenges, as it requires maintaining and updating transition estimates for a joint state space that grows exponentially with the number of arms $N$ and the size of the state space increase. This compounded complexity results in the curse of dimensionality, making such methods intractable for large-scale, high-dimensional RMAB problems.

Most RMAB settings impose budget constraints on the selected arms, adding another layer of complexity. Under such constraints, the definition of regret becomes less straightforward, as it must account for both the non-stationary dynamics and the time-varying nature of the Lagrange multiplier. This means the underlying constrained MDP is no longer directly compatible with the original algorithm proposed in \cite{cheung2020reinforcement}. This challenge is further compounded by the fact that non-stationary RMABs require a dynamically changing Bellman equation, where the value function, policy, and Lagrange multiplier all adapt over time. At this point, a natural question arises:

\begin{quote}
    \centering \emph{Can we design a simple, computationally tractable algorithm that can operate in a high-dimensional non-stationary restless multi-armed bandit environment?}
\end{quote}

This paper provides a positive answer to this question. We first define a new regret proxy that approximates the dynamic regret using a stationary value function. By fixing the transition dynamics, policy, and Lagrange multiplier at their current value, we reduce the problem to a sequence of simpler, stationary subproblems, while 
 capturing the essential optimality gap and significantly reducing computational complexity, bridging the gap between theory and practice for non-stationary RMABs. 
 
Our approach builds on two key insights:

\begin{itemize}[leftmargin=*]
    \item \textbf{Arm-Specific Sliding Windows}: Instead of treating the entire RMAB as a single large MDP, we maintain transition estimates separately for each arm.  This significantly reduces the computational burden, enabling more localized learning and avoiding the exponential blow-up in the joint state space. This arm-specific estimate also allows arm-specific sliding windows depending on the transition function changes.
    \item \textbf{Non-stationary optimistic Whittle index}: We apply the minimax theorem to compute the optimistic non-stationary transition probabilities in the confidence intervals estimated by sliding windows. The optimistic transition probabilities later are used to compute the non-stationary Whittle index and update the pulling threshold. 
\end{itemize}
These techniques significantly reduce both the dimensionality and computational overhead, making the approach scalable even in large, high-dimensional, non-stationary environments. 

Crucially, unlike treating the entire RMAB as a single large MDP with state space $\mathcal{S}^N$, which leads to regret bounds that scale exponentially with $N$, our approach maintains only a quadratic dependence on $N$ by decoupling the joint optimization across arms. This effectively avoids the curse of dimensionality by leveraging the weak dependency between arms, providing a scalable and theoretically grounded solution to non-stationary RMABs. Furthermore, our approach achieves a regret bound of $\widetilde{\mathcal{O}}(N^2 |\mathcal{S}|^{\frac{1}{2}} B^{\frac{3}{4}} \delta^{-\frac{1}{2}} T^{\frac{3}{4}})$ with probability $1 - \delta$, where $N$ is the number of arms and $|\mathcal{S}|$ represents the shared state space of the independent MDPs.  This is a significant improvement over the naive approach of treating the RMAB as a single non-stationary MDP to use the regret bound in~\cite{cheung2020reinforcement}, which would result in a much larger regret bound of $\widetilde{\mathcal{O}}(|\mathcal{S}|^{\frac{2N}{3}}B^{\frac{1}{4}} T^{\frac{3}{4}})$ due to the exponential state space. 

In particular, when budget $B =0$, i.e., the transition is stationary, our regret bound reduces to the best known regret bound $\Tilde{O}(T^{\frac{1}{2}})$ for the stationary non-episodic and episodic RMAB settings as shown in \cite{wang2023optimistic,jung2019regret,jung2019thompson}, effectively bridging the gap between the stationary case and the stationary case.

\section{Related Work}
\paragraph{Offline restless multi-armed bandits (RMABs)}
Restless multi-armed bandits (RMABs)~\cite{whittle1988restless} generalize multi-armed bandits (MABs)~\cite{slivkins2019introduction,kuleshov2014algorithms} by allowing multiple ($K$) arms to be pulled at the same time, where arms are modeled as independent MDPs.
With the MDP parameters given, optimizing the RMAB as an offline optimization problem is known to be PSPACE-hard~\cite{papadimitriou1999complexity}. As an approximate solution, the Whittle index policy~\cite{weber1990index} is shown to be asymptotically optimal under the indexability condition~\cite{liu2008restless,liu2010indexability}. Many Whittle index-based approaches~\cite{robledo2022qwi,biswas2021learn,fu2019towards} were proposed to solve RMABs in polynomial time.

\paragraph{Online learning in MDPs}
When the MDP parameters are unknown a priori, the goal is to interact with the environment to learn the MDP parameters and simultaneously maximize the cumulative rewards.
In stationary infinite-horizon MDPs, algorithms~\cite{auer2008near,auer2006logarithmic,bartlett2012regal} with $\tilde{O}(S \sqrt{A T})$ upper bound on regret under different conditions and a lower bound $\Omega(\sqrt{SAT})$ were shown, where $S$ and $A$ denotes the state and action space size.
In non-stationary infinite-horizon MDPs with limited changes in MDP parameters, \cite{cheung2020reinforcement,wei2021non} propose the \textit{SWUCRL2-CW} algorithm with $O(T^{\frac{3}{4}})$ regret upper bound and no lower bound is known.
When the number of changes is known, ~\cite{gajane2018sliding} shows a $O(T^{\frac{2}{3}})$ regret bound.
For episodic non-stationary MDPs, ~\cite{mao2021near} shows a $\tilde{O}(T^{\frac{2}{3}})$ regret with a lower bound $\Omega(T^{\frac{2}{3}})$. A detailed summary and a master approach that handle non-stationary MDPs can be found in ~\cite{wei2021non}.


\paragraph{Online learning in RMABs}
For RMABs, due to the exponentially large state and action space in RMABs, directly applying the results from MDPs can incur exponentially large regret bound.
For stationary RMABs, when the rewards are collected only by the arms pulled, \cite{liu2012learning} demonstrate a $O(\log T)$ regret bound. When rewards are collected from all arms, the problems become significantly hard where \cite{jung2019regret} use a Thompson sampling approach to achieve a $\tilde{O}(\sqrt{T})$ Bayesian regret bound, while \cite{wang2023optimistic} use UCB to show a $\tilde{O}(N |S| \sqrt{T \log T})$ frequentist regret bound. 
For non-stationary RMABs, there has been no prior work on the regret analysis. Our work provides the first regret bound that scales linearly in the number of arms in non-stationary RMABs.

\section{Problem Statement}
\subsection{Non-Stationary Restless Bandits} 
We consider an RMAB problem with $N$ restless arms over a time horizon of $T$ time steps. At each time step $t \in [T]$, the learner selects at most $K$ arms to activate, where $K \leq N$ is the activation budget, after observing the state $s_{t,i}$ of each arm $i$.
Each arm $i \in [N]$ is associated with a non-stationary discounted MDP $\mathcal{M}_i:=(\mathcal{S},\mathcal{A},\gamma,R_i,\{P_{t,i}\}_{t=1}^T)$, where $\mathcal{S}$ is the state space, and $\mathcal{A}=\{0,1\}$ denotes the action space (1 for active and 0 for passive), $\gamma\in [0,1)$ is the discount factor, $R_{i}(s_{t,i},a_{t,i})$ denotes the state-action-dependent per-arm reward function, where are stationary and known to the learner, and  $P_{t,i}:\mathcal{S}\times \mathcal{A}\rightarrow \Delta(\mathcal{S})$\footnote{Throughout the paper, for a set $\mathcal{Z}$, we use $\Delta(\mathcal{Z})$ to denote the set of all the probability distributions over $\mathcal{Z}$.} determines the state transition probability under non-stationarity, \ie 
\begin{align}
    P_{t,i}(s' \mid s, a) = \Pr(s_{t+1,i} = s' \mid s_{t,i} = s, a_{t,i} = a).
    \label{eq:def of transition}
\end{align}

\paragraph{Budget constrain for state transition} To capture the non-stationarity of the environment, we introduce a finite variation budget that limits the extent to which transition probabilities can shift over time. Specifically, we define the per-step variation for each arm $i$ at time $t$ as:
\begin{align}
    B_{t,i} &:= \max_{s,a}\lVert P_{t+1,i}(\cdot|s,a) - P_{t,i}(\cdot|s,a)\rVert_1, \label{eq:budget}
\end{align}
where this budget captures the maximum possible change in the transition dynamics between consecutive time steps. The total allowable variation across all arms and time steps is then given by $B := \sum_{t=1}^T \sum_{i=1}^N B_{t,i}$. For a stationary MDP, this variation is zero, meaning the transition matrices, $P_{t,i}$ remain identical across all time steps, reflecting time-invariant dynamics.

\paragraph{Nonstationary policy} A stochastic policy in the multi-armed restless bandit setting is a probability distribution over actions given the state. Formally, a stochastic policy for arm $i$ at time $t$ is given by: $\pi_{t,i}(a_{t,i}|s_{t,i}) = \Pr(a_{t,i}|s_{t,i})$, where $a_{t,i}\in \{0,1\}$ indicates whether the arm is activated $(1)$ or passive $(0)$. To respect the activation budget, this policy must satisfy the constraint: $\sum_{i=1}^{N} \pi_{t,i}(1 | s) = K, \forall s\in \mathcal{S}$, ensuring that no more than $K$ arms are activated at any given time in expectation.             

\paragraph{Joint MDP Formulation}
To analyze regret effectively, it is convenient to represent the entire RMAB system as a single, unified MDP. This joint MDP formulation aggregates the individual state, action, and transition dynamics of each arm as follows: (i) Global state: $\bs_t := (s_{t,1}, s_{t,2}, \dots, s_{t,N})$, representing the concatenated state vector of all arms at time $t$; (ii)  Global action $\ba_t := (a_{t,1}, a_{t,2}, \dots, a_{t,N})$, the concatenated action vector; (iii) Joint transition dynamic: Given by the product of individual arm transitions, $\bP_t(\bs_{t+1} \mid \bs_t, \ba_t) := \prod_{i=1}^N P_{t,i}(s_{t+1,i}|s_{t,i},a_{t,i})$, capturing the joint transition probability of the global state at time $t$; (iv) Joint reward function: The sum of the individual rewards, $\bR_t(\bs_t,\ba_t) := \sum_{i=1}^{N} R_i(s_{t,i}, a_{t,i}) $ , representing the total reward at time $t$; (v) Joint policy: The product of individual arm policies, given by
\begin{align}
    \pi_t(\ba_t \mid \bs_t) = \prod_{i=1}^{N} \pi_{t,i}(a_{t,i} | s_{t,i}),
\end{align}
ensuring that the overall policy captures the joint decision-making process across all arms. This joint MDP representation is critical for defining and analyzing regret in non-stationary RMABs, as it captures the full complexity of the system while allowing for efficient decomposition into more manageable subproblems.

\subsection{Optimal Policy and Regret}
To evaluate the performance of the proposed algorithm in non-stationary MDPs, we rely on value functions, which capture the expected total discounted reward of a given policy $\pi$ over time. Specifically, the value function $V^{\pi}_{\bP,\blambda}$ and the action-value function $Q^{\pi}_{\bP,\blambda}$ under a joint policy $\pi$ and transition matrix $\bP$ with dual variable $\blambda$ are defined as:

\begin{align}
    V^{\pi}_{\bP,\blambda}(\bs) :=& \mathbb{E}_{(\bs,\ba) \sim (\bP,\pi)}\left[ \sum_{t\in \mathbb{N}} \gamma^{t-1} \underbrace{\left(\bR(\bs_t,\ba_t) - \blambda\left(\pi(\bs_t)^\top \mathbbm{1} - K\right)\right)}_{:=\bR_{\blambda}(\bs_t,\ba_t)} | \bs_1=\bs\right] \label{eq:relax V} \\
    Q^{\pi}_{\bP,\blambda}(\bs,\ba) :=& \mathbb{E}_{(\bs,\ba) \sim (\bP,\pi)}\left[ \sum_{t\in \mathbb{N}} \gamma^{t-1} \left(\bR(\bs_t,\ba_t) - \blambda\left(\pi(\bs_t)^\top \mathbbm{1} - K\right)\right) | \bs_1=\bs, \ba_1=\ba\right] \label{eq:relax Q}
\end{align}

The optimal primal policy $\pi^*_t$ and the corresponding dual variable $\lambda^*_t$ at time step $t$ are obtained by solving the min-max problem:
\begin{align}
    \lambda^*_t, ~\pi^*_{t} = \arg\min_{\lambda} \max_{\pi} V^{\pi}_{\bP_t,\lambda}(\bs)  \label{eq:min-max}
\end{align}
Given access to the optimal dual variable $\lambda^*_t$, we can estimate the performance of a policy $\pi_t$ by its the value function $V^{\pi_{t}}_{\bP_t,\blambda^*_t}(\bs_t)$, which is always upper bounded by the optimal value $V^{\pi^*_{t}}_{\bP_t,\blambda^*_t}(\bs_t)$.
The regret over a time horizon $T$ is then defined as the cumulative gap between the optimal value function and the value function under the chosen policy, given by
\begin{align}
    \text{Reg}(T) := \sum_{t=1}^T \left( V^{\pi^*_{t}}_{\bP_t,\blambda^*_t}(\bs_t) - V^{\pi_{t}}_{\bP_t,\blambda^*_t}(\bs_t) \right) \label{eq:def of regret}
\end{align}
In the non-stationary setting, this regret definition captures both the immediate suboptimality of the current policy and the long-term impact of time-varying transition dynamics. However, unlike stationary MDPs where Bellman equations provide a clear recursive relationship for value functions, the non-stationary nature here introduces additional complexity as the transition matrix, policy, and value functions themselves evolve over time. To address this, we approximate this dynamic regret using a stationary perspective, enabling a more tractable analysis while still capturing the key sources of suboptimality.

\begin{remark}
\normalfont
According to \cite{liu2020regret}, in the non-episodic RL setting, the environment provides an initial state $ \bs_1 $, and subsequent states are generated based on the transition dynamics and the learned policy. Unlike episodic RL, non-episodic RL no longer resets to an initial state, making reliable recovery impossible. Instead, the performance is evaluated over the trajectory $ \{\bs_1, \dots, \bs_T\} $ induced by the learned policy. The goal of RL is to learn a policy that minimizes the optimality gap, defined as the difference between the cumulative expected discounted return of the optimal policy and that of the current policy.
\end{remark}


\section{Methodology}

\begin{figure}
    \centering
    \includegraphics[width=1.0\linewidth]{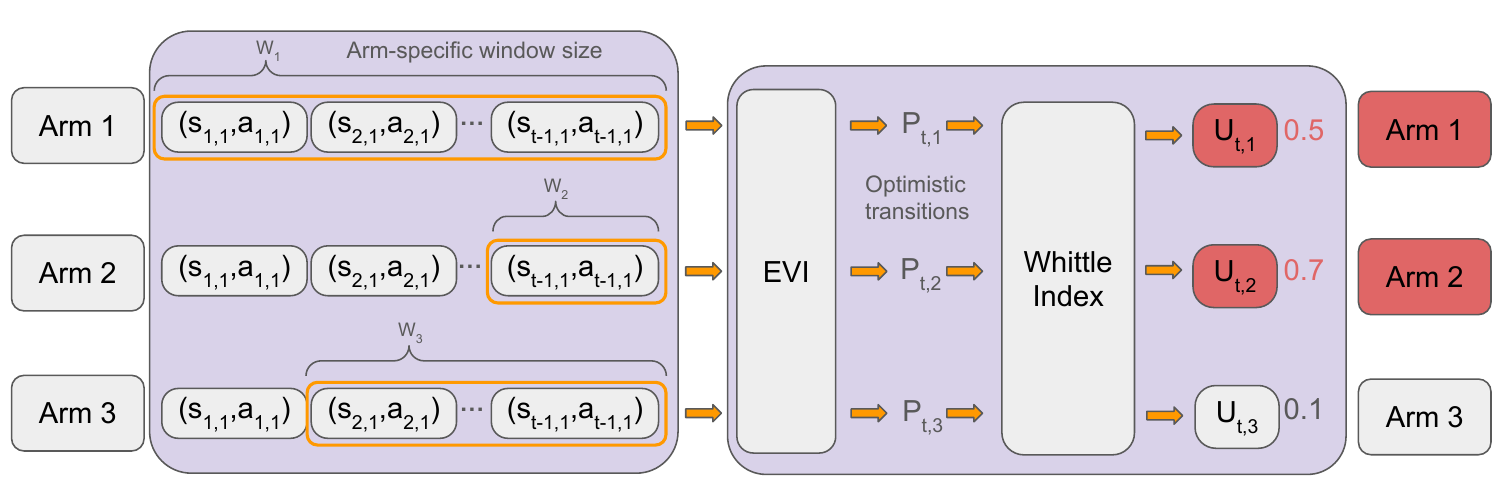}
    \caption{The pipeline of \rmab. First, each arm uses a tailored window size to estimate its empirical transition dynamics, effectively capturing recent, relevant data. These estimated transitions are then processed by the EVI module, which applies UCB to compute optimistic dynamics, incorporating the current dual variable to balance long-term rewards and budget constraints. Finally, Whittle indices are calculated from these optimistic transitions, and the top $K$ arms are selected for activation at each time step.}
    \label{figure:algorithm}
\end{figure}

The main idea behind \rmab; is to leverage arm-specific sliding windows to estimate the optimal Lagrange multiplier and policy defined in (\ref{eq:min-max}). This approach is specifically designed to handle the challenges of time-varying dynamics and budget constraints while maintaining tight regret bounds. The overall algorithm pipeline is illustrated in Figure \ref{figure:algorithm}. We first introduce the sliding window mechanism with UCB, which serves as the foundation for accurately estimating optimistic transition probabilities. This estimation is then refined through EVI to obtain policies that balance immediate rewards against long-term gains, guided by adaptive budget constraints.

\subsection{Sliding Windows with UCB}

We adopt the sliding window framework from \cite{cheung2020reinforcement} to estimate the optimistic transitions. Specifically, each arm $i$ is assigned a sliding window parameter $W_i \in \mathbb{N}$, which determines the number of recent time steps used to estimate the arm's unknown transition dynamics. This adaptive windowing strategy captures the time-varying nature of each arm's dynamics without requiring full historical data. Additionally, an exploration parameter $\{\eta_i\}_{i=1}^N \geq 0$ is introduced to incorporate optimism in the transition estimates, ensuring sufficient exploration under uncertainty.
To construct the empirical estimate of each arm's transition function, we define the counter for the number of times each state-action pair $(s, a)$ appears within the corresponding window $W_i$ for arm $i$ at time $t$:
\begin{align}
    N_{t,i}^{+}(s,a) =& \sum_{q=(t-1-W_i)^+}^{t-1} \mathbbm{1} (s_{q,i} = s, a_{q,i} = a),
\end{align}
where the operator \( x^+ \) is defined as \( x^+ = \max\{x,1\} \).  Next, we define the observed average transition probability as:
\begin{align}
    \widetilde{P}_{t,i}(s'|s,a) =& \frac{\sum\limits_{q=(t-1-W_i)^+}^{t-1} \mathbbm{1} (s_{q,i} = s, a_{q,i} = a, s_{q+1,i}  = s')}{N_{t,i}^{+}(s,a)},
\end{align}
which serves as an estimator for the true average transition probability:
\begin{align}
    \hat{P}_{t,i}(s'|s,a) =& \frac{\sum\limits_{q=(t-1-W_i)^+}^{t-1} P_{q,i}(s'|s,a) \mathbbm{1} (s_{q,i} = s, a_{q,i} = a)}{N_{t,i}^{+}(s,a)}.
\end{align}

We then define the confidence interval for the transition probability and the corresponding "good event" as follows:

\begin{align}
    H_{t,i}(s,a,\eta) =& \left\{ P \in \Delta(\mathcal{S}) \; \big| \; \lVert P(\cdot|s,a)  - \widetilde{P}_{t,i}(\cdot|s,a)\rVert \leq \text{rad}_{t,i}(s,a) + \eta_i\right\} \label{eq:H}\\
    \xi_{i} =& \left\{ \hat{P}_{t,i}(s'|s,a) \in H_{t,i}(s,a,0), \forall s \in \mathcal{S}, a\in \mathcal{A}, t \in [T] \right\},\label{eq:xi}
\end{align}
The confidence interval is defined by: 
\begin{align}
    \text{rad}_{t,i}(s,a) = \sqrt{2|S|\log(|S||A|T/\delta)/N^+_{t,i}(s,a)}. \label{eq:radius}
\end{align}

\subsection{Optimistic Model Construction and Policy Optimization}

Once the empirical transition probabilities are estimated, the next step is to construct an optimistic model that maximizes the expected long-term reward within these confidence intervals. This is achieved by solving the following optimization problem for each arm $i$: $\max_{\pi,P\in H_{t,i}} V^\pi_{P,\lambda,i}$, which is a function of $\lambda$, for each arm $i$ using Extended Value Iteration (EVI) \citep{auer2008near}. In EVI, presented in Algorithm \ref{alg:vi}, the transition dynamics $P$ is treated as an additional control variable, selected from $H_{t,i}$, thereby yielding an extended policy. The procedure of EVI is presented as follows: For all $s \in \mathcal{S}$, we initialize:  
\begin{align}
    \bar{Q}^{(0)}_{\lambda,i}(s,a) &= 0.
\end{align}
Then, for each iteration $u$, we update:
\begin{align}
    \bar{Q}^{(u+1)}_{\lambda,i}(s,a) &= -\lambda a + R_{t,i}(s,a) + \gamma\max_{P \in H_{t,i}(s,a,\eta)}\sum_{s'}P(s'|s,a)\max_{a} \bar{Q}^{(u)}_{\lambda,i}(s,a).
\end{align}
After a total of $U$ iterations, we solve for the dual variable by minimizing $\lambda$ over the learned action-value function:
\begin{align}
    \lambda_t = \argmin_{\lambda}\sum_{i=1}^N\bar{Q}^{(U)}_{\lambda,i}(s_t,a_t).
\end{align}

Finally, we define the {\rmab} policy adopted at time $t$ using the following steps:
\begin{align}
    \bar{P}_{t,i}(\lambda) &= \argmax_{P \in H_{t,i}} \bar{Q}^{(U)}_{\lambda,i}(s,a), \\
    \bar{P}_{t,i} &= \argmin_{\lambda \geq 0} \bar{Q}_{\bar{P}_{t,i}(\lambda),\lambda,i}(s,a).
\end{align}

The deterministic policy is determined as:
\begin{align}
    \pi_{t,i}(s_{t,i}) &=
    \begin{cases}
        1, & \text{if } \bar{Q}_{\bar{P}_{t,i},\lambda_t,i}(s_{t,i},1) \geq \bar{Q}_{\bar{P}_{t,i},\lambda_t,i}(s_{t,i},0), \\
        0, & \text{otherwise}. 
    \end{cases} \label{eq:rmab policy}
\end{align}
This structured approach ensures that each arm independently estimates and optimizes its policy while adhering to a global budget constraint, thereby effectively solving the non-stationary RMAB problem.


\begin{algorithm}
    \caption{\rmab:\;Whittle Index Policy for Non-stationary RMABs}
    \begin{algorithmic}[1]
    \label{alg:rmab}
        \STATE {\bfseries Input: } $N$, $\eta$, $W$, $\lambda_0$, $\pi_0$, $U$, $\kappa$
        \FOR{$t=0,1,\cdots$}
            \STATE Take action $a_{t,i}$ according to $\pi_{t,i}(s_{t,i}) $ as defined in (\ref{eq:rmab policy}), $\forall i \in [N]$
            \STATE Observe reward $R_{t,i}$ and next state $ s_{t+1,i}\sim P_{t,i}(\cdot|s_{t,i},a_{t,i}), \forall i \in [N]$
            \STATE Set $\lambda_t = \lambda_0$ and $H_{t,i}(\cdot,\cdot)$ according to (\ref{eq:H}) and $\bar{Q}^{(0)}_{\lambda_t,i}(\cdot,\cdot) = 0 , \forall i\in[N]$,
            \WHILE{until $\bar{Q}$, $\lambda_t$ converge}
            \STATE $\left\{\bar{Q}_{\lambda_t,i}\right\}_{i=1}^N \leftarrow \text{EVI}\left(\left\{\bar{Q}^{(0)}_{\lambda_t,i}\right\}_{i=1}^N, \left\{H_{t,i}(\cdot,\cdot,\eta)\right\}_{i=1}^N,U\right)$ 
            \STATE $\lambda_{t+1} = \argmin_{\lambda_t > 0} \sum_{i=1}^N \bar{Q}_{\lambda_t,i}$
            \ENDWHILE
            \ENDFOR
    \end{algorithmic}
\end{algorithm}

\begin{algorithm}
    \caption{Extended Value Iteration (EVI)}
    \begin{algorithmic}[1]
        \STATE {\bfseries Input:} $\left\{\bar{Q}^{(0)}_{\lambda,i}\right\}_{i=1}^N, \left\{H_{t,i}(s,a,\eta)\right\}_{i=1}^N,U$
        \FOR{$u=1,2,\cdots,U$}
            \STATE $\bar{Q}^{(u+1)}_{\lambda,i}(s,a) \! = \! -\lambda a \! + \! R_{t,i}(s,a) + \gamma\max\limits_{P \in H_{t,i}}\sum\limits_{s'}P(s'|s,a)\max\limits_{a'} \bar{Q}^{(u)}_{\lambda,i}(s',a'), \quad \forall s,a, \forall i\in[N]$
        \ENDFOR
        \STATE $\bar{P}_{t,i} = \argmax\limits_{P \in H_{t,i}} \sum\limits_{s'}P(s'|s,a)\max\limits_{a'} \bar{Q}^{(U)}_{\lambda,i}(s',a'), \quad \forall s,a, \forall i\in[N]$
        \STATE Return $\left\{\bar{Q}^{(U)}_{\lambda,i}\right\}_{i=1}^N$
    \end{algorithmic}
    \label{alg:vi}
\end{algorithm}

\begin{remark}
    By min-max inequality, we have
    \begin{align}
        \min_{\lambda} \max_{\pi,P} V^\pi_{\lambda,P,i} \geq \max_P \underbrace{\min_{\lambda} \max_{\pi} V^\pi_{\lambda,P,i}}_{\text{True whittle index}},
        \label{eq:algo}
    \end{align}
    where (\ref{eq:algo}) learns a upper bound of solution of the min-max problem presented in (\ref{eq:min-max}).
\end{remark}


\section{Theoretical Analysis}

\subsection{Main Results}
In this section, we present the theoretical guarantees for \rmab; as outlined in Algorithm \ref{alg:rmab}. Our main result, Theorem \ref{theorem:regret bound}, that establishes a sub-linear regret bound for \rmab, demonstrating its long-term effectiveness in non-stationary RMAB problems. 
\begin{theorem}
\label{theorem:regret bound}
With probability at least $1-6\delta' - N/\delta$, we have
\begin{align}
    \text{Reg}_{\blambda^*_t}(T) =&\; \frac{2NBW(N + U_{\lambda}K)}{\eta(1-\gamma)}  + \frac{2N(1+U_\lambda)\gamma}{(1-\gamma)^2}\sqrt{T\log{\frac{1}{\delta'}}} + \frac{\gamma N(U_{\lambda}+1)}{p_{\text{min}}(1-\gamma)^2}\sqrt{T\log{\frac{1}{\delta'}}} \nonumber \\
        &\; + \frac{\gamma N(U_{\lambda}+1)}{p_{\text{min}}(1-\gamma)^2}\left(TN\eta + N\cdot\sqrt{\rad{T}}\left(\frac{T}{\sqrt{W}}\cdot \left(\sqrt{2}+1\right)\sqrt{\abs{\mathcal{S}}\abs{\mathcal{A}}}\right)\right)\label{eq:regret eq-1}.
\end{align}
The detailed proof is provided in Appendix \ref{appendix:proof of regret}.
\begin{remark}
    To optimize the regret bound, we can choose $W_i^* = |S| T^{1/2} (\sum_{t=1}^TB_{t,i})^{-\frac{1}{2}}, \eta_i^*=\sqrt{BW_i^*/T}$ for arm $i$. The regret reduces to $\text{Reg}_{\blambda^*_t}(T) = \widetilde{\mathcal{O}}(N^2 |\mathcal{S}|^{\frac{1}{2}} B^{\frac{1}{4}} \delta^{-\frac{1}{2}} T^{\frac{3}{4}})$.
\end{remark}

\begin{remark}
    When the environment is stationary, i.e., $B=0$, we can choose $W = T, \eta = 1/T$ and the regret bound reduces to $\widetilde{\mathcal{O}}(N^2 |S| T^{\frac{1}{2}})$, which matches the best known non-episodic and episodic regret in stationary restless bandits shown in \cite{wang2023optimistic,jung2019regret,jung2019thompson}.
\end{remark}
\end{theorem}

\subsection{Proof Sketch}
In this section, we establish the regret bound for \rmab, as described in Algorithm \ref{alg:rmab}, by providing a structured overview of the proof. The analysis begins by leveraging the optimality condition for our policy, derived from the EVI step and minimization over Lagrange multipliers, which satisfies the min-max property (\ref{eq:algo}). Since the value function depends on the policy, Lagrange multiplier, and transition dynamics, it is crucial to replace the optimal policy in the first term of the regret expression with the learned policy. This substitution ensures that the only remaining difference between the two value functions is the transition dynamics, simplifying the analysis. Given this alignment, we invoke Lemma \ref{lemma:value replacement} to replace the oracle term with the learned policy and its associated optimistic transition dynamics come from EVI.

Then, To bridge the gap between these optimistic transitions and the ground truth state transitions, we apply Lemma \ref{lemma:regret' bound} that bounds this difference based on the cumulative transition variation. This approach leverages the predefined variation budget, ensuring the overall regret remains within the desired bounds. To introduce the associated lemmas and theorems, we first present two key assumptions that form the foundation of our regret analysis.
\begin{assumption}
\label{assumption:Upper bound of lambda}
For all $t \in [T]$ and some positive constant $U_{\lambda} \geq 0$, we have $\lambda_t^* \leq U_{\lambda}$. 
\end{assumption}

\begin{lemma}
\label{lemma:Upper bound of V}
    For all $t \in [T]$, and $\bP \in \mathcal{P}$, Then, for all $\bs \in \mathcal{S}^N$, we have
    \begin{align}
        |V^{\pi_t}_{\bP,\blambda_t^*}(\bs)| \leq \upV,
    \end{align}
    where $\mathcal{P}$ is the set of joint transitions for all arms. The proof is provided in Appendix \ref{appendix:proof use useful lemma}.
\end{lemma}

To derive an upper bound on the value function with Lagrange relaxation, it is essential to impose this assumption to ensure that the optimal Lagrange multipliers remain uniformly bounded across the entire time horizon. Without this constraint, achieving sub-linear regret becomes fundamentally intractable. This is because, under certain conditions, the optimal Lagrange multiplier can diverge to infinity if the budget constraint approaches its boundary, as dictated by the complementary slackness condition. Such unbounded multipliers would cause the regret to scale uncontrollably, undermining any meaningful performance guarantee.

\begin{assumption}
    \label{assumption:p_min}
    For the non-stationary state transitions over time, we assume a uniform lower bound on all non-zero transition probabilities, defined as $p_{\text{min}} := \min_{(s,a,s',t,i) \not\in P_{0}} P_{t,i}(s'|s,a)$, where $P_0$ is the set of all zero-probability transitions given by 
    \begin{align}
        P_{0} := \left\{ (s,a,s't,i)| P_{t,i}(s'|s,a) = 0, \forall s,s' \in \mathcal{S}, a \in \mathcal{A}, t \in [T] \text{ and } i \in [N]\right\}
    \end{align}
\end{assumption}
This assumption ensures that non-zero probabilities in the true state transition have a minimum value, which is critical for controlling the magnitude of transition ratios that arise in importance sampling. Specifically, it only requires that the observed transition probabilities remain above a known constant, without demanding explicit knowledge of the constant $p_{\text{min}}$ or the exact composition of the zero-transition $P_0$. Importantly, this lower bound only applies to transitions that have already been observed, which are inherently non-zero. This makes the assumption more practical and widely applicable, aligning with similar conditions commonly adopted in model-based reinforcement learning literature \cite{kumar1982optimal}. 

\subsection{Regret Substitution}
To facilitate the regret analysis, we introduce the following value function substitution:
\begin{lemma}
    \label{lemma:value replacement}
    Given $\pi_t$, and the estimated transitions $\bar{P}_t$ learned by Algorithm \ref{alg:rmab}, $\bP_{t}$, $\lambda_t^*$ are the true transition and optimal dual variable, respectively. With probability at least $1 - N/\delta$, We have 
    \begin{align}
        V^{\pi_t}_{\bar{\bP}_t,\blambda^*_t}(\bs_t)\geq V^{\pi^*_t}_{\bP_t,\blambda^*_t}(\bs_t).
    \end{align}
    \begin{proof} 
        We have
        \begin{align}
            V^{\pi_t}_{\bar{\bP}_t,\blambda^*_t}(\bs_t) \underset{(i)}{\geq} V^{\pi_t}_{\bar{\bP}_t,\blambda_t}(\bs_t) \underset{(ii)}{\geq} V^{\pi^*_t}_{\bP_t,\blambda_t}(\bs_t) \underset{(iii)}{\geq} V^{\pi^*_t}_{\bP_t,\blambda^*_t}(\bs_t),
        \end{align}
        where $(i)$ holds due to $\blambda_t$ is the minimizer of $V^{\pi_t}_{\bar{\bP}_t,\blambda}(\bs_t)$, $(ii)$ follows from the fact that $\pi_t, \bar{\bP_t}$ maximize $V^{\pi}_{\bP,\blambda_t}(\bs), \forall \bs \in \mathcal{S}^{N}$, and $(iii)$ holds because $\blambda^*_t$ is the minimizer of $V^{\pi^*_t}_{\bP_t,\blambda}(\bs_t)$. 
    \end{proof}
\end{lemma}

Lemma \ref{lemma:value replacement} provides a value function substitution that allows us to reformulate the original regret definition into the form $\text{\normalfont Reg}'(T)$, which is defined as follows:

\begin{align}
    \text{\normalfont Reg}'(T) := \sum_{t=1}^{T} \left( V^{\pi_t}_{\bar{\bP}_t,\blambda_t^*}(\bs_t) - V^{\pi_t}_{\bP_t,\blambda^*_t}(\bs_t) \right) \label{eq: regret'}
\end{align}
This reformulation isolates the impact of transition dynamics on the regret by focusing solely on the difference between the optimistic transitions and the true, underlying dynamics, significantly simplifying the analysis. This gap can then be bounded using the radius of the confidence interval defined in (\ref{eq:radius}), providing a tighter, more tractable expression for the regret bound.

\subsection{Counting Bad Event Occurrences}
\label{section:construction of Q_T}
We begin by analyzing the sliding window technique to quantify the time steps where bad events occur. We define the bad event set $Q_T$ as the set of time steps satisfying the following two conditions:
\begin{itemize}
    \item There exist $s \in \mathcal{S}, a \in \mathcal{A}$, and $i \in [N]$ such that $P_{t,i}(\cdot|s,a) \notin H_{t,i}(s, a,\eta)$, indicating that the true transition dynamics fall outside the confidence set.
    \item For each $t' \in Q_T$, the gap between consecutive bad event occurrences is larger than the window size. $\ie, t - t' > W$.
\end{itemize}
Next, we construct an extended set $\widetilde{Q}_T$ that includes all elements of $Q_T$ along with any time steps within a window $W$ of any element in $Q_T$, defined as $\widetilde{Q}_T = Q_T \cup \left\{t \in [T]: \exists t' \in Q_T, t - t' \in [0,W]\right\}$. This extension constructs a pessimistic set connected to $Q_T$ to capture all adverse events. With this setup, we state the following lemma:

\begin{lemma}
\label{lemma:size of Q}
    Conditioned on $\{\xi_i\}_{i=1}^N$, we have $\forall t$ such that $\exists i \in [N], P_{t,i}(\cdot|s,a) \notin H_{t,i}(s, a,\eta)$, is in $\widetilde{Q}_T$ and $|\widetilde{Q}_T| \leq WB/\eta$, the proof is provided in Appendix \ref{appendix:proof of lemma:size of Q}.
\end{lemma} 

\subsection{Upper bound of Reg\texorpdfstring{$'(T)$}{'(T)}}
The following lemma provides an upper bound of $\text{\normalfont Reg}'(T)$ under the good event.
\begin{lemma}
\label{lemma:regret' bound}
    Condition on $\left\{ P_{t,i}(\cdot|s,a) \in H_{t,i}(s,a,\eta), \forall (s,a,i) \right\}_{t=1}^T$, with probability at least $1-6\delta' - N/\delta$, we have
    \begin{align}
        \text{\normalfont Reg}'(T) \leq &\;\frac{2N(1+U_\lambda)\gamma}{(1-\gamma)^2}\sqrt{T\log{\frac{1}{\delta'}}} + \frac{\gamma N(U_{\lambda}+1)}{p_{\text{min}}(1-\gamma)^2}\sqrt{T\log{\frac{1}{\delta'}}} \nonumber \\
        &\; + \frac{\gamma N(U_{\lambda}+1)}{p_{\text{min}}(1-\gamma)^2}\left(TN\eta + N\cdot\sqrt{\rad{T}}\left(\frac{T}{\sqrt{W}}\cdot \left(\sqrt{2}+1\right)\sqrt{\abs{\mathcal{S}}\abs{\mathcal{A}}}\right)\right).
    \end{align}    
\end{lemma}
The proof is provided in Appendix \ref{appendix:proof of lemma:regret' bound}. Combining Lemma \ref{lemma:value replacement} and Lemma \ref{lemma:regret' bound}, we can bound the regret for the case where the true transition dynamics lie within the confidence interval. Additionally, Lemma \ref{lemma:size of Q} enables us to account for the regret contribution from bad events, leveraging the bounded property of the value function to capture the impact of these time steps.







\section{Limitations}
This work has two main limitations. First, the regret analysis relies on a uniform bound for the optimal Lagrange multipliers (Assumption \ref{assumption:Upper bound of lambda}), which, while necessary for sub-linear regret, may be restrictive in settings with widely varying multipliers. Second, the regret is analyzed using a relaxed formulation (Lemma \ref{lemma:value replacement}) that substitutes the optimal value function with the learned one, potentially introducing an approximation gap that may not fully capture non-stationary dynamics. Future work can address these issues for a more precise regret characterization.

\section{Conclusion}
In this work, we developed a scalable framework for non-stationary RMABs that effectively balances computational efficiency with theoretical rigor. By introducing a new regret proxy based on a stationary value function approximation, we addressed the core challenges of non-stationarity, including time-varying transition dynamics, adaptive policies, and changing Lagrange multipliers. Our method significantly reduces the dimensionality of the problem by decoupling the joint optimization across arms, leveraging arm-specific sliding windows and a global, structured update for the Lagrange multiplier. This design not only captures the essential optimality gap of non-stationary RMABs but also provides a scalable foundation for practical deployment. Our approach opens new avenues for efficiently managing complex, resource-constrained decision-making problems in dynamic environments.

\bibliography{reference}
\bibliographystyle{unsrt}

\newpage

\newpage
\appendix
\section{Useful Lemmas}
\label{appendix:proof use useful lemma}
\begin{lemma}[Azuma-Hoeffding inequality]
\label{lemma:Azuma–Hoeffding inequality}
Let $(X_i,\mathcal{F}_i)_{i=1,\cdots,t}$ to be a martingale satisfying $X_i < M, \forall i\in[t]$, for some positive constant $M$. Then, for any $n\in [t]$, we have 
\begin{align}
    \Pr\left(X_n \leq 2M\sqrt{n\log{\frac{1}{\delta}}}\right) \geq 1-\delta
\end{align}
\end{lemma}

\begin{lemma}[Lemma 7, \cite{cheung2020reinforcement}]
\label{lemma:good event out of Q_T}
    For any $t \notin \widetilde{Q}_T$, $\forall (s,a,i)$, we have 
    \begin{align}
        P_{t,i}(\cdot|s,a) \in H_{t,i}(s,a,\eta).
    \end{align}
\end{lemma}
\begin{proof}
    We follow a similar proof as of Lemma 7 in \cite{cheung2020reinforcement}, based on the construction of $\widetilde{Q}_T$ provided in Section \ref{section:construction of Q_T}. Suppose there exists a time step $t \notin \widetilde{Q}_T$ such that $P_{t,i}(\cdot \mid s,a) \notin H_{t,i}(s, a, \eta)$ for some state $s$, action $a$ and arm $i$. According to the first condition in the construction of $Q_T$, such a $t$ should have been included in $Q_T$. Otherwise, we would have $|t - t'| \leq W_i$ for some $t' \in Q_T$, which contradicts the construction of $\widetilde{Q}_T$.
\end{proof}
\begin{lemma}[Lemma 1, \cite{cheung2020reinforcement}]
    \label{lemma:prob of good event}
    For all $i \in [N]$, we have $\Pr(\xi_i) \geq 1-\delta/2$.
\end{lemma}

\begin{lemma}
    For all $t \in [T]$, and $\bP \in \mathcal{P}$, Then, for all $\bs \in \mathcal{S}^N$, we have
    \begin{align}
        |V^{\pi_t}_{\bP,\blambda_t^*}(\bs)| \leq \upV,
    \end{align}
    where $\mathcal{P}$ is the set of joint transitions for all arms. 
\begin{proof}
    By definition of $V^{\pi_t}_{\bP,\blambda^*_t}(\bs)$, we have 
    \begin{align}
        V^{\pi_t}_{\bP,\blambda^*_t}(\bs) = & \;\mathbb{E}_{(\bs,\ba) \sim (\bP,\pi_t)}\left[ \sum_{t\in \mathbb{N}} \gamma^{t-1} \left(\bR(\bs_t,\ba_t) - \blambda^*_t\left(\pi_t(\bs_t)^\top \mathbbm{1} - K\right)\right) | \bs_1=\bs\right] \label{eq:Upper bound of V eq-1}\\
        \leq & \; \mathbb{E}_{(\bs,\ba) \sim (\bP,\pi_t)}\left[ \sum_{t\in \mathbb{N}} \gamma^{t-1} \left(\bR(\bs_t,\ba_t) + U_{\lambda}N\right) | \bs_1=\bs\right] \label{eq:Upper bound of V eq-2}\\
        \leq & \; \upV\label{eq:Upper bound of V eq-3},
    \end{align}
    where (\ref{eq:Upper bound of V eq-2}) holds by $\pi_t(\bs)^\top \mathbbm{1} \leq K$ and Assumption \ref{assumption:Upper bound of lambda}, and (\ref{eq:Upper bound of V eq-3}) caused by $0\leq R_i(s,a)\leq 1, \forall s \in \mathcal{S}, a \in \mathcal{A}$, and $i \in [N]$. Therefore, the statement holds.  
\end{proof}
\end{lemma}

\section{Proof of Lemma \ref{lemma:size of Q}}
\label{appendix:proof of lemma:size of Q}
\newtheorem*{lemma58}{Lemma 5.8}
\begin{lemma58}
    Conditioned on $\{\xi_i\}_{i=1}^N$, we have $\forall t$ such that $\exists i \in [N], P_{t,i}(\cdot|s,a) \notin H_{t,i}(s, a,\eta)$, is in $\widetilde{Q}_T$ and $|\widetilde{Q}_T| \leq WB/\eta$, the proof is provided in Appendix \ref{appendix:proof of lemma:size of Q}.
\end{lemma58} 
\begin{proof}
    For all $t \in Q_T$, we claim that exist $(s,a,i)$ and $\exists t'\in [(t-W)^+,t]$, we have 
    \begin{align}
        \norm{P_{t,i}(\cdot|s,a) - P_{t',i}(\cdot|s,a)} \geq \eta. \label{eq:size of Q eq-1}
    \end{align}
    Assume $\forall (s,a,i,t'\in [(t-W)^+,t]), \norm{P_{t,i}(\cdot|s,a) - P_{t',i}(\cdot|s,a)} < \eta$, and consider
    \begin{itemize}
        \item Case $1$, if $N_{t,i}(s,a) = 0$, we have
        \begin{align}
            \norm{P_{t,i}(\cdot|s,a) - \widetilde{P}_{t,i}(\cdot|s,a)} = 0 \leq \text{rad}_{t,i}(s,a) \leq \text{rad}_{t,i}(s,a) + \eta, \label{eq:size of Q eq-2}
        \end{align}
        which contradicts to the fact $\forall t \in Q_T, P_{t,i}(\cdot|s,a) \notin H_{t,i}(s,a,\eta)$. 
        \item Case $2$, if $N_{t,i}(s,a) > 0$, we have
        \begin{align}
            &\norm{P_{t,i}(\cdot|s,a) - \hat{P}_{t,i}(\cdot|s,a)} \nonumber \\
            \leq\;& \frac{\sum\limits_{q=(t-1-W)^+}^{t-1} \norm{P_{t,i}(\cdot|s,a) - P_{q,i}(\cdot|s,a)} \cdot \mathbbm{1} (s_{q,i} = s, a_{q,i} = a)}{N_{t,i}^{+}(s,a)} \label{eq:size of Q eq-3}\\
            \leq\;& \eta, \label{eq:size of Q eq-4}
        \end{align}
        where (\ref{eq:size of Q eq-3}) holds by triangle inequality, and (\ref{eq:size of Q eq-4}) holds due to the assumption.  
        Then, we have
        \begin{align}
            \norm{P_{t,i}(\cdot|s,a) - \widetilde{P}_{t,i}(\cdot|s,a)}\leq& \norm{\widetilde{P}_{t,i}(\cdot|s,a) - \hat{P}_{t,i}(\cdot|s,a)} + \norm{P_{t,i}(\cdot|s,a) - \hat{P}_{t,i}(\cdot|s,a)} \label{eq:size of Q eq-5}\\
            \leq & \;\text{rad}_{t,i}(s,a) + \eta, \label{eq:size of Q eq-6}
        \end{align}
        where (\ref{eq:size of Q eq-5}) follows from triangle inequality, and (\ref{eq:size of Q eq-6}) holds due to the condition on $\{\xi_i\}_{i=1}^N$ and (\ref{eq:size of Q eq-4}). Then, we have (\ref{eq:size of Q eq-6}) contradicts the fact $\forall t \in Q_T, P_{t,i}(\cdot|s,a)\notin H_{t,i}(s,a,\eta)$. 
    \end{itemize}
    Finally, we have
    \begin{align}
        B =&  \sum_{t=1}^T \max_{s,a,i}\norm{P_{t+1,i}(\cdot|s,a) - P_{t,i}(\cdot|s,a)} \label{eq:size of Q eq-7}\\
        \geq & \sum_{t\in Q_T} \sum_{t' \in [(t - 1 - W)^+ , (t-1)^+]}  \max_{s,a,i} \norm{P_{t'+1,i}(\cdot|s,a) - P_{t',i}(\cdot|s,a)} \label{eq:size of Q eq-8}\\
        \geq & \sum_{t\in Q_T} \max_{s,a,i, t' \in [(t - 1 - W)^+ , (t-1)^+]}\norm{P_{t'+1,i}(\cdot|s,a) - P_{t',i}(\cdot|s,a)} \label{eq:size of Q eq-9}\\
         \geq & \abs{Q_T}\eta,\label{eq:size of Q eq-10}
    \end{align}
    where (\ref{eq:size of Q eq-7}) follows from the budget definition in (\ref{eq:budget}), while (\ref{eq:size of Q eq-8}) holds because the time steps in $Q_T$ are spaced at least $W$ apart. (\ref{eq:size of Q eq-9}) follows by replacing the summation with a maximum operation, and (\ref{eq:size of Q eq-10}) holds by (\ref{eq:size of Q eq-1}). We finished the proof.
\end{proof}

\section{Proof of Lemma \ref{lemma:regret' bound}}
\label{appendix:proof of lemma:regret' bound}
Recall that
\begin{align}
    \text{\normalfont Reg}'(T) := \sum_{t=1}^{T} \left( V^{\pi_t}_{\bar{\bP}_t,\blambda_t^*}(\bs_t) - V^{\pi_t}_{\bP_t,\blambda^*_t}(\bs_t) \right).
\end{align}
The following lemma shows the upper bound of $\text{\normalfont Reg}'(T)$.
\newtheorem*{lemma59}{Lemma 5.9}
\begin{lemma59}
    Condition on $\left\{ P_{t,i}(\cdot|s,a) \in H_{t,i}(s,a,\eta), \forall (s,a,i) \right\}_{t=1}^T$, with probability at least $1-6\delta' - N/\delta$, we have
    \begin{align}
        \text{\normalfont Reg}'(T) \leq &\;\frac{2N(1+U_\lambda)\gamma}{(1-\gamma)^2}\sqrt{T\log{\frac{1}{\delta'}}} + \frac{\gamma N(U_{\lambda}+1)}{p_{\text{min}}(1-\gamma)^2}\sqrt{T\log{\frac{1}{\delta'}}} \nonumber \\
        &\; + \frac{\gamma N(U_{\lambda}+1)}{p_{\text{min}}(1-\gamma)^2}\left(TN\eta + N\cdot\sqrt{\rad{T}}\left(\frac{T}{\sqrt{W}}\cdot \left(\sqrt{2}+1\right)\sqrt{\abs{\mathcal{S}}\abs{\mathcal{A}}}\right)\right). 
    \end{align}    
\end{lemma59}
\begin{proof}
    By the definition of $\text{\normalfont Reg}'(T)$, we have
    \begin{align}
       \text{\normalfont Reg}'(T) =&\; \sum_{t=1}^{T} \left( V^{\pi_t}_{\bar{\bP}_t,\blambda_t^*}(\bs_t) - V^{\pi_t}_{\bP_t,\blambda^*_t}(\bs_t) \right) \\
       = & \;\gamma \sum_{t=1}^{T}\left[ \E_{s'\sim \bar{\bP}_t(\cdot|\bs_t,\ba_t)}\left[V^{\pi}_{\bar{\bP}_t,\blambda}(\bs')\right]  - \E_{s'\sim \bP_t(\cdot|\bs_t,\ba_t)}\left[V^{\pi}_{\bP_t,\blambda}(\bs')\right]\right] \label{eq:regret' eq-1} \\
        = &\; \gamma \sum_{t=1}^{T}\bigg[ \underbrace{\E_{s'\sim \bP_t(\cdot|\bs_t,\ba_t)}\left[V^{\pi}_{\bar{\bP}_t,\blambda}(\bs') - V^{\pi}_{\bP_t,\blambda}(\bs')\right] - \left( V^{\pi}_{\bar{\bP}_t,\blambda}(\bs_{t+1}) - V^{\pi}_{\bP_t,\blambda}(\bs_{t+1})\right)}_{:=B_1} \nonumber \\
        &\; \underbrace{-\E_{s'\sim \bP_t(\cdot|\bs_t,\ba_t)}\left[V^{\pi}_{\bar{\bP}_t,\blambda}(\bs')\right]  + V^{\pi}_{\bar{\bP}_t,\blambda}(\bs_{t+1}) }_{:=B_2}  \nonumber \\
        &\;  \underbrace{ + \E_{s'\sim \bar{\bP}_t(\cdot|\bs_t,\ba_t)}\left[{V^{\pi}_{\bar{\bP}_t,\blambda}(\bs')}\right]  - V^{\pi}_{\bP_t,\blambda}(\bs_{t+1})}_{:=B_3}\bigg]\label{eq:regret' eq-2}   \\
        \leq &\; \frac{2N(1 + U_{\lambda})\gamma}{1-\gamma}\sqrt{T\log{\frac{1}{\delta'}}} + B_3,  \label{eq:regret' eq-3}  
    \end{align}
    where (\ref{eq:regret' eq-1}) holds due 
 to the immediate rewards for both value function, $V^{\pi_t}_{\bar{\bP}_t,\blambda^*_t}(\bs_t)$ and $ V^{\pi_{t}}_{\bP_t,\blambda^*_t}(\bs_t)$, are the same under the identical policy $\pi_t$, and (\ref{eq:regret' eq-3}) holds with probability at least $1-2\delta'$ by applying Azuma-Hoeffding inequality, presented in Lemma \ref{lemma:Azuma–Hoeffding inequality}, on the martingale difference sequences, $B_1$ and $B_2$, with upper bound $M=(N + U_{\lambda}N)/(1-\gamma)$. For $B_3$, we have
    \begin{align}
        B_3 =&\;\gamma \sum_{t=1}^{T-1}\left( \E_{s'\sim \bar{\bP}_t(\cdot|\bs_t,\ba_t)}\left[{V^{\pi}_{\bar{\bP}_t,\blambda}(\bs')}\right] {\color{black} - V^{\pi}_{\bP_t,\blambda}(\bs_{t+1})}\right) \label{eq:regret' eq-4}\\
        =&\;\gamma \sum_{t=1}^{T-1}\left( \E_{s'\sim \bP_t(\cdot|\bs_t,\ba_t)}\left[ \frac{\bar{\bP}_t(\bs_{t+1}|\bs_t,\ba_t)}{\bP_t(\bs_{t+1}|\bs_t,\ba_t)}V^{\pi}_{\bar{\bP}_t,\blambda}(\bs')\right]  - V^{\pi}_{\bP_t,\blambda}(\bs_{t+1})\right)\label{eq:regret eq-10} \\
        =&\;\gamma \sum_{t=1}^{T-1}\bigg( \underbrace{\E_{s'\sim \bP_t(\cdot|\bs_t,\ba_t)}\left[ \frac{\bar{\bP}_t(\bs_{t+1}|\bs_t,\ba_t)}{\bP_t(\bs_{t+1}|\bs_t,\ba_t)}V^{\pi}_{\bar{\bP}_t,\blambda}(\bs')\right]  - \frac{\bar{\bP}_t(\bs_{t+1}|\bs_t,\ba_t)}{\bP_t(\bs_{t+1}|\bs_t,\ba_t)}V^{\pi}_{\bar{\bP}_t,\blambda}(\bs_{t+1})}_{B_4} \nonumber \\
        & \;+  \underbrace{\frac{\bar{\bP}_t(\bs_{t+1}|\bs_t,\ba_t)}{\bP_t(\bs_{t+1}|\bs_t,\ba_t)}V^{\pi}_{\bar{\bP}_t,\blambda}(\bs_{t+1}) - V^{\pi}_{\bP_t,\blambda}(\bs_{t+1}) }_{B_5}\bigg)\label{eq:regret eq-11}. 
        \end{align}
        For $B_4$, the same as $B_1$ and $B_2$, by Azuma-Hoeffding inequality (Lemma \ref{lemma:Azuma–Hoeffding inequality}) with $M = p_{\text{min}}^{-1} \cdot \upV$, and probability at least $1-\delta'$, we have
        \begin{align}
            B_4 \leq \frac{\gamma N(U_{\lambda}+1)}{p_{\text{min}}(1-\gamma)}\sqrt{T\log{\frac{1}{\delta'}}} \label{eq:regret eq-12}.
        \end{align}
        For $B_5$, we have
        \begin{align}
            B_5 = &\;\gamma \sum_{t=1}^{T-1} \left( \frac{\bar{\bP}_t(\bs_{t+1}|\bs_t,\ba_t)}{\bP_t(\bs_{t+1}|\bs_t,\ba_t)}V^{\pi}_{\bar{\bP}_t,\blambda}(\bs_{t+1}) - V^{\pi}_{\bP_t,\blambda}(\bs_{t+1}) \right) \label{eq:regret eq-13}\\
            = &\; \gamma \sum_{t=1}^{T-1} \left( \frac{\bar{\bP}_t(\bs_{t+1}|\bs_t,\ba_t)}{\bP_t(\bs_{t+1}|\bs_t,\ba_t)} -1 \right)V^{\pi}_{\bar{\bP}_t,\blambda}(\bs_{t+1}) + \gamma\sum_{t=1}^{T-1} \left( V^{\pi}_{\bar{\bP}_t,\blambda}(\bs_{t+1}) - V^{\pi}_{\bP_t,\blambda}(\bs_{t+1}) \right) \label{eq:regret eq-14}\\
            \leq &\; \frac{\gamma N(U_{\lambda}+1)}{p_{\text{min}}(1-\gamma)} \underbrace{\sum_{t=1}^{T-1} \left(\bar{\bP}_t(\bs_{t+1}|\bs_t,\ba_t) - \bP_t(\bs_{t+1}|\bs_t,\ba_t)\right)}_{B_6} + \underbrace{\gamma\sum_{t=1}^{T-1} \left( V^{\pi}_{\bar{\bP}_t,\blambda}(\bs_{t+1}) - V^{\pi}_{\bP_t,\blambda}(\bs_{t+1}) \right)}_{B_7}\label{eq:regret eq-15},
        \end{align}
        where (\ref{eq:regret eq-15}) holds by Assumption \ref{assumption:p_min} and Lemma \ref{lemma:Upper bound of V}. For $B_6$, by the property that $\{P_{t,i} \in H_{t,i}\}_{i=1}^N$ for all $t \notin \widetilde{Q}_T$, we further have
        \begin{align}
         B_6 \leq & \;\sum_{t=1}^T\sum_{s,a}\sum_{i=1}^N \left( \text{rad}_{t,i}(s,a)+\eta\right) \mathbbm{1}(s_{t,i}=s,a_{t,i}=a) \label{eq:regret eq-16}\\
         \leq&  \; TN\eta + \sqrt{\rad{T}}\cdot \sum_{i=1}^N\sum_{s,a}\sum_{t=1}^T \frac{\mathbbm{1}(s_{t,i}=s,a_{t,i}=a)}{\sqrt{ N_{t,i}(s,a)}} \label{eq:regret eq-17}
    \end{align}
    Then, by partitioning the total time $T$ into $M_i(T)$ episodes, each of width $W_i$ for arm $i$, we denote the starting time step of the $m$-th episode as $\tau_i(m)$ for arm $i$, we have 
    \begin{align}
        \sum_{s,a}\sum_{t=1}^T \frac{\mathbbm{1}(s_{t,i}=s,a_{t,i}=a)}{\sqrt{ N_{t,i}(s,a)}} =& \sum_{s,a}\sum_{m\in M_i(T)}\sum_{t=\tau_i(m)}^{\tau_i(m+1)-1} \frac{\mathbbm{1}(s_{t,i}=s,a_{t,i}=a)}{\sqrt{N_{t,i}(s,a)}} \label{eq:regret eq-18}\\
        \leq & \sum_{s,a}\sum_{m\in M_i(T)}\sum_{t=\tau_i(m)}^{\tau_i(m+1)-1}\frac{\mathbbm{1}(s_{t,i}=s,a_{t,i}=a)}{\sqrt{ \hat{N}_{t,i}(s,a)}} \label{eq:regret eq-19} \\
        \leq & \sum_{s,a}\sum_{m\in M_i(T)}  \left(\sqrt{2}+1\right)\sqrt{\hat{N}_{\tau_i(m+1)-1,i}(s,a)} \label{eq:regret eq-20}\\
        \leq &\; \frac{T}{W_i}\cdot \left(\sqrt{2}+1\right) \sqrt{\abs{\mathcal{S}}\abs{\mathcal{A}}\sum_{s,a}\hat{N}_{\tau_i(m+1)-1,i}(s,a)} \label{eq:regret eq-21}\\
        =&\; \frac{T}{\sqrt{W_i}}\cdot \left(\sqrt{2}+1\right)\sqrt{\abs{\mathcal{S}}\abs{\mathcal{A}}}\label{eq:regret eq-22}
    \end{align}
    where (\ref{eq:regret eq-19}) holds by defining $\hat{N}_{t,i}(s,a) := \sum_{q=\max\{(t-1-W_i),\tau_i(m)\}}^{t-1} \mathbbm{1} (s_{q,i} = s, a_{q,i} = a)$, if $t \in m$-th episode, (\ref{eq:regret eq-20}) holds by Lemma 19 in \cite{auer2008near}, and (\ref{eq:regret eq-21}) holds by Jenson's inequality.
    Substituting (\ref{eq:regret eq-22}) into (\ref{eq:regret eq-17}), and denoting $W := \min_{i\in[N]} W_i$, we have
    \begin{align}
        B_6 \leq TN\eta + N\cdot\sqrt{\rad{T}}\left(\frac{T}{\sqrt{W}}\cdot \left(\sqrt{2}+1\right)\right) \sqrt{\abs{\mathcal{S}}\abs{\mathcal{A}}}\label{eq:regret eq-23}
    \end{align}
    To simplify the notation, we aggregate the error terms from equations (\ref{eq:regret' eq-3}), (\ref{eq:regret eq-11}), (\ref{eq:regret eq-12}), (\ref{eq:regret eq-15}), and (\ref{eq:regret eq-23}). With probability at least $1-3\delta$, we denote the sum of the error terms as a constant $C$, defined as:
    \begin{align}
        C := &\;\frac{2N(1+U_\lambda)\gamma}{1-\gamma}\sqrt{T\log{\frac{1}{\delta'}}} + \frac{\gamma N(U_{\lambda}+1)}{p_{\text{min}}(1-\gamma)}\sqrt{T\log{\frac{1}{\delta'}}} \nonumber \\
        &\; + \frac{\gamma N(U_{\lambda}+1)}{p_{\text{min}}(1-\gamma)}\left(TN\eta + N\cdot\sqrt{\rad{T}}\left(\frac{T}{\sqrt{W}}\cdot \left(\sqrt{2}+1\right)\sqrt{\abs{\mathcal{S}}\abs{\mathcal{A}}}\right)\right). \label{eq:regret eq-24}
    \end{align}
    Using this definition, with probability at least $1-3\delta$, the regret can be expressed as:
    \begin{align}
        \text{\normalfont Reg}'(s_t,\bar{P}_t, P_t, \pi_t, \lambda^*_t) = &\; \sum_{t=1}^{T} \left( V^{\pi_t}_{\bar{\bP}_t,\blambda_t^*}(\bs_t) - V^{\pi_t}_{\bP_t,\blambda^*_t}(\bs_t) \right) \\
        \leq &\; C + \gamma\sum_{t=1}^{T-1} \left( V^{\pi_{t}}_{\bar{\bP}_t,\blambda^*_t}(\bs_{t+1}) - V^{\pi_{t}}_{\bP_t,\blambda^*_t}(\bs_{t+1}) \right) \label{eq:regret eq-25}\\
         = &\;\frac{2N(1+U_\lambda)\gamma}{(1-\gamma)^2}\sqrt{T\log{\frac{1}{\delta'}}} + \frac{\gamma N(U_{\lambda}+1)}{p_{\text{min}}(1-\gamma)^2}\sqrt{T\log{\frac{1}{\delta'}}} \nonumber \\
        &\; + \frac{\gamma N(U_{\lambda}+1)}{p_{\text{min}}(1-\gamma)^2}\left(TN\eta + N\cdot\sqrt{\rad{T}}\left(\frac{T}{\sqrt{W}}\cdot \left(\sqrt{2}+1\right)\sqrt{\abs{\mathcal{S}}\abs{\mathcal{A}}}\right)\right). \label{eq:regret eq-26}
    \end{align}
    where (\ref{eq:regret eq-25}) holds by (\ref{eq:regret eq-15}) and (\ref{eq:regret eq-26}) follows by repeatedly applying the similar process (\ref{eq:regret' eq-1})-(\ref{eq:regret eq-24}) to the term $\gamma\sum_{t=1}^{T-1} \left( V^{\pi_{t}}_{\bar{\bP}_t,\blambda^*_t}(\bs_{t+1}) - V^{\pi_{t}}_{\bP_t,\blambda^*_t}(\bs_{t+1}) \right)$ for $T$ iterations, and using the summation formula for the geometric series. We have finished the proof.

    
\end{proof}

\section{Proof of Theorem \ref{theorem:regret bound}}
\label{appendix:proof of regret}
\newtheorem*{theorem51}{Theorem 5.1}
\begin{theorem51}
With probability at least $1-6\delta' - N/\delta$, we have
\begin{align}
    \text{Reg}_{\blambda^*_t}(T) =&\; \frac{2NBW(N + U_{\lambda}K)}{\eta(1-\gamma)}  + \frac{2N(1+U_\lambda)\gamma}{(1-\gamma)^2}\sqrt{T\log{\frac{1}{\delta'}}} + \frac{\gamma N(U_{\lambda}+1)}{p_{\text{min}}(1-\gamma)^2}\sqrt{T\log{\frac{1}{\delta'}}} \nonumber \\
        &\; + \frac{\gamma N(U_{\lambda}+1)}{p_{\text{min}}(1-\gamma)^2}\left(TN\eta + N\cdot\sqrt{\rad{T}}\left(\frac{T}{\sqrt{W}}\cdot \left(\sqrt{2}+1\right)\sqrt{\abs{\mathcal{S}}\abs{\mathcal{A}}}\right)\right)
\end{align}
\end{theorem51}
\begin{proof}
By the definition of regret in (\ref{eq:def of regret}) and the following definitions
\begin{align}
    A_1 :=&\; \frac{2NBW(N + U_{\lambda}K)}{\eta(1-\gamma)} \\
    A_2 :=&\; \frac{2N(1+U_\lambda)\gamma}{(1-\gamma)^2}\sqrt{T\log{\frac{1}{\delta'}}} + \frac{\gamma N(U_{\lambda}+1)}{p_{\text{min}}(1-\gamma)^2}\sqrt{T\log{\frac{1}{\delta'}}} \nonumber \\
        &\; + \frac{\gamma N(U_{\lambda}+1)}{p_{\text{min}}(1-\gamma)^2}\left(TN\eta + N\cdot\sqrt{\rad{T}}\left(\frac{T}{\sqrt{W}}\cdot \left(\sqrt{2}+1\right)\sqrt{\abs{\mathcal{S}}\abs{\mathcal{A}}}\right)\right)
\end{align}
we have
    \begin{align} 
        &\PR{\text{Reg}_{\blambda^*_t}(T)> A_1 + A_2}\nonumber \\
        =\;& \underbrace{\PR{\sum_{t=1}^T \left( V^{\pi^*_{t}}_{\bP_t,\blambda^*_t}(\bs_t) - V^{\pi_{t}}_{\bP_t,\blambda^*_t}(\bs_t) \right)\geq A_1 + A_2 \mid \{\xi_i\}_{i=1}^N}}_{I_1}  \cdot \PR{\{\xi_i\}_{i=1}^N} \nonumber \\
        \;& + \underbrace{\PR{\sum_{t=1}^T \left( V^{\pi^*_{t}}_{\bP_t,\blambda^*_t}(\bs_t) - V^{\pi_{t}}_{\bP_t,\blambda^*_t}(\bs_t) \right)> A_1 + A_2 \mid \text{not }\{\xi_i\}_{i=1}^N}}_{I_2}\cdot \PR{\text{not }\{\xi_i\}_{i=1}^N} \label{eq:regret eq-2}\\
        \leq\;& \underbrace{\PR{\sum_{t\in\widetilde{Q}_T}\left( V^{\pi^*_{t}}_{\bP_t,\blambda^*_t}(\bs_t) - V^{\pi_{t}}_{\bP_t,\blambda^*_t}(\bs_t) \right) > A_1 \mid \{\xi_i\}_{i=1}^N}}_{I_3} \cdot \PR{\{\xi_i\}_{i=1}^N}\nonumber \\
        & + \PR{\sum_{t\notin\widetilde{Q}_T} \left( V^{\pi^*_{t}}_{\bP_t,\blambda^*_t}(\bs_t) - V^{\pi_{t}}_{\bP_t,\blambda^*_t}(\bs_t) \right)> A_2\mid \{\xi_i\}_{i=1}^N} \cdot \PR{\{\xi_i\}_{i=1}^N} + \PR{\text{not }\{\xi_i\}_{i=1}^N}\label{eq:regret eq-3} ,\\
        \leq\;& \PR{\sum_{t\notin\widetilde{Q}_T} \left( V^{\pi^*_{t}}_{\bP_t,\blambda^*_t}(\bs_t) - V^{\pi_{t}}_{\bP_t,\blambda^*_t}(\bs_t) \right)> A_2\mid \{\xi_i\}_{i=1}^N} \cdot \PR{\{\xi_i\}_{i=1}^N} + \PR{\text{not }\{\xi_i\}_{i=1}^N}\label{eq:regret eq-4} ,
    \end{align}
    where (\ref{eq:regret eq-3}) follows from the union bound for $I_1$ and the fact that $I_2 \leq 1$, while (\ref{eq:regret eq-4}) holds because  
    $I_3 = 0 $ due to Lemma \ref{lemma:size of Q}, Lemma \ref{lemma:Upper bound of V}, and the fact that the size of $\widetilde{Q}_T$ is $W$ times  $\abs{Q_T}$. Then, we proceed with the RHS of (\ref{eq:regret eq-4}). Given that $\PR{\{\xi_i\}_{i=1}^N} \leq 1$ and $I_4 \geq 0$, and define $\mathcal{B}_t := \left\{ P_{t,i}(\cdot|s,a) \in H_{t,i}(s,a,\eta), \forall (s,a,i) \right\}$, we have
    \begin{align} 
        &\PR{\sum_{t\notin\widetilde{Q}_T} \left( V^{\pi^*_{t}}_{\bP_t,\blambda^*_t}(\bs_t) - V^{\pi_{t}}_{\bP_t,\blambda^*_t}(\bs_t) \right)> A_2\mid \{\xi_i\}_{i=1}^N} \cdot \PR{\{\xi_i\}_{i=1}^N} + \PR{\text{not }\{\xi_i\}_{i=1}^N}\nonumber \\
        \leq\;& \PR{\sum_{t\notin\widetilde{Q}_T}\left( V^{\pi^*_{t}}_{\bP_t,\blambda^*_t}(\bs_t) - V^{\pi_{t}}_{\bP_t,\blambda^*_t}(\bs_t) \right) > A_2 \mid \{\xi_i\}_{i=1}^N} + \PR{\text{not }\{\xi_i\}_{i=1}^N}\nonumber \\
        & + \underbrace{\PR{\sum_{t\in\widetilde{Q}_T}\left( V^{\pi^*_{t}}_{\bP_t,\blambda^*_t}(\bs_t) - V^{\pi_{t}}_{\bP_t,\blambda^*_t}(\bs_t) \right) > A_2\mid \{\xi_i\}_{i=1}^N \cap \mathcal{B}_t}}_{I_4} \label{eq:regret eq-5} \\
        \leq \;& \PR{\sum_{t\notin\widetilde{Q}_T}\left( V^{\pi_{t}}_{\bar{\bP}_t,\blambda^*_t}(\bs_t) - V^{\pi_{t}}_{\bP_t,\blambda^*_t}(\bs_t) \right) > A_2 \mid \{\xi_i\}_{i=1}^N \cap \mathcal{B}_t} + \PR{\text{not }\{\xi_i\}_{i=1}^N}\nonumber \\
        & + \PR{\sum_{t\in\widetilde{Q}_T}\left( V^{\pi_{t}}_{\bar{\bP}_t,\blambda^*_t}(\bs_t) - V^{\pi_{t}}_{\bP_t,\blambda^*_t}(\bs_t) \right) > A_2 \mid \{\xi_i\}_{i=1}^N \cap \mathcal{B}_t} \label{eq:regret eq-6} \\
        \leq\;& 
        2\cdot\PR{\text{Reg}'(1,T)> A_2 \mid \{\xi_i\}_{i=1}^N \cap \mathcal{B}_t} + \PR{\text{not }\{\xi_i\}_{i=1}^N} \label{eq:regret eq-7} \\
        \leq\;& 
        6\delta'  + N/{\delta}\label{eq:regret eq-8} ,
    \end{align}
    where (\ref{eq:regret eq-6}) since, for all $t \notin \widetilde{Q}_T$, the event $ \mathcal{B}_t$ follows from Lemma \ref{lemma:good event out of Q_T}, (\ref{eq:regret eq-7}) holds by defining $\text{\normalfont Reg}'(1,T) := \sum_{t=1}^{T} \left( V^{\pi_t}_{\bar{\bP}_t,\blambda^*_t}(\bs_t) - V^{\pi_{t}}_{\bP_t,\blambda^*_t}(\bs_t) \right)$, and (\ref{eq:regret eq-8}) follows from Lemma \ref{lemma:prob of good event} and Lemma \ref{lemma:regret' bound}.

\end{proof}

\end{document}